\documentclass[a4paper]{article}
\usepackage{amssymb}
\usepackage{amsmath}
\usepackage[hyphens]{url} 
\usepackage[utf8]{inputenc}
\usepackage{mathtools}  
\usepackage{amsthm}
\usepackage{color}
\usepackage{nicefrac}
\usepackage{tcolorbox}
\usepackage{natbib}
 \newtheorem{thm}{Theorem}[section]
 \newtheorem{algo}{Algorithm}
 \newtheorem{prop}{Proposition}[section]
 \newtheorem{lem}{Lemma}[section]
 \newtheorem{cor}{Corollary}[section]
 \newtheorem{ass}{Assumption}
 \newtheorem{exm}{Example}[section]
 \newtheorem{dfn}{Definition}
 \newtheorem{rem}{Remark}

\def \e{\epsilon}
\def\la{\langle}
\def\ra{\rangle}
\def\R{\mathbb{R}}
\def\E{\mathbb{E}}
\def\P{\mathbb{P}}
\def\Rn{\mathbb{R}^n}

\def\R{\mathbb{R}}
\title{Discrete Choice Multi-Armed Bandits}

\author{Emerson Melo \thanks{Department of Economics,  Indiana University Bloomington; e-mail:  emelo@iu.edu}  \and 
David Müller \thanks{Department of Mathematics, Chemnitz University of Technology; e-mail: david.mueller@mathematik.tu-chemnitz.de}  } 
\providecommand{\keywords}[1]{\textbf{\textit{Keywords---}} #1}

\begin{document}

\maketitle
\vspace{18mm} \setcounter{page}{1}

\begin{abstract}	
This paper establishes a connection between a category of discrete choice models and the realms of online learning and multiarmed bandit algorithms. Our contributions can be summarized in two key aspects. Firstly, we furnish sublinear regret bounds for a comprehensive family of algorithms, encompassing the Exp3 algorithm as a particular case. Secondly, we introduce a novel family of adversarial multiarmed bandit algorithms, drawing inspiration from the generalized nested logit models initially introduced by \citet{wen:2001}. These algorithms offer users the flexibility to fine-tune the model extensively, as they can be implemented efficiently due to their closed-form sampling distribution probabilities.
To demonstrate the practical implementation of our algorithms, we present numerical experiments, focusing on the stochastic bandit case.\end{abstract}

\section{Introduction}
In this paper we analyze how discrete choice models can be utilized in order to design online optimization and multiarmed bandit algorithms for the experts setting. In doing so, we build upon the gradient based prediction algorithms (GBPA) introduced by \cite{abernethy2016perturbation}, where the authors derive algorithms based on convex potential functions. In their framework, the online learning algorithm employs an update step that corresponds to the gradient of a specified convex potential function.   Other classes of algorithms like Follow the Perturbed Leader and and Follow the Regularized leader can be analyzed in the GBPA framework. In  the first part of this paper we show how gradient based prediction algorithms for the experts setting can be derived from surplus functions of random utility models (RUM). This connection has recently been made by \cite{melo2021learning}. For that reason some of the presented results resemble results from \cite{melo2021learning}. However the focus of our paper lies on the computational aspects of the new designed algorithms. More precisely, we incorporate algorithmic aspects of discrete choice models examined by \cite{prox} into the GBPA framework from \cite{abernethy2016perturbation}. This provides sublinear regret bounds for a whole family of algorithms including the Exp3 algorithm as special case. By presenting and reanalyzing the online optimization case we are able to achieve a better regret bound than \cite{abernethy2016perturbation} for the Gumbel smoothing scenario. Additionally, we compare the algorithmic aspects of the multinomial logit and the nested logit surplus function.  \\[1 em]
The second part of our work finally discusses the multiarmed bandit scenario. The results are partly build on the analysis of part 2 and the work by \cite{abernethy2016perturbation}. In particular, we derive a new family of adverserial multiarmed bandit algorithms based on generalized nested logit models introduced by \cite{wen:2001}.  This is a remarkable generalization of the Exp3 algorithm, see for example \cite{cesa2006prediction}. Furthermore, we prove sublinear expected regret bounds for the family of algorithms relying on the loss-only setting. The new algorithms yield  a non-independent change of the sampling probabilities. Hence, these algorithms provide the user with the possibility to tune the model more thoroughly. At the same time, the presented algorithms can be implemented in a highly efficient manner, as the probabilities of the sampling distribution are given in closed form.   \\[1 em]
The third part of the paper presents numerical experiments of  the stochastic bandit case. On the contrary to adverserial  bandits, the former are characterized by distribution following rewards. Our numerical experiments show closing regret for the nested logit bandit algorithms. Moreover, we provide examples to show that the algorithm is able to outperform the classical Exp3-Algorithm. In particular, a nested logit based algorithm provides structures in the Exploration/Exploitation trade-off that can not be achieved by the algorithm based on the multinomial logit model.  \\[1 em] 
 {\bfseries Notation}  Our notation is quite standard. By $\mathbb{R}^n $ we denote the space of n-dimensional vectors, where  the vectors $x = \left(x^{(1)}, x^{(2)}, \ldots, x^{(n)}\right)^T $ are column vectors. For $x \in \Rn$ we write $x^{-(i)} \in \R^{n-1}$ meaning that the $i$-th component of $x$ is missing. Analogously, we write $x^{-(i,j)} \in \R^{n-2}$ meaning that both $i$-th and $j$-th components of $x$ are missing. Using the latter, we write with some abuse of notation:
 \[
 x = \left(x^{-(i,j)}, x^{(j)}, x^{(i)}
 \right).
 \] 

 We denote by $e_j \in \R^n$ the $j$-th coordinate vector of $\R^n$ and write  $e$ for the vector of an appropriate dimension whose components are equal to one. Similar we write $\mathbf{0}$ for the vector of an appropriate dimension whose components are equal to zero. For a vector $x \in \R^n$ we write $e^x$ for the exponential operation of all the components, i.\,e. 
 \[
 e^x = \left(e^{x{(1)}}, e^{x{(2)}}, \ldots, e^{x{(i)}}, \ldots, e^{x{(n)}} \right)^T.
 \]
 With this convention the following holds:
 \[
 e^{\mathbf{0}} = \left(e^{0}, \ldots, e^{0}, \ldots, e^{0} \right)^T = \left(1, \ldots, 1, \ldots, e^1 \right)^T = e^T.
 \]
 By $\R^n_+$ we denote the set of all vectors with nonnegative components.
We introduce the standard inner product in $\mathbb{R}^n$:
\[
\left\langle x,y\right\rangle = \sum\limits_{i=1}^{n} x^{(i)} y^{(i)}. 
\]
 If $y > 0$ we define the vector division:
 \[
 \frac{x}{y} = \left(\frac{x^{(1)}}{y^{(1)}}, \ldots, \frac{x^{(n)}}{y^{(n)}}\right)^T.
 \]
 For $x \in \mathbb{R}^n $ we use the norms
\[
\|x\|_1 = \sum\limits_{i=1}^{n} |x^{(i)}|, \quad 
\|x\|_2 = \sqrt{\sum\limits_{i=1}^{n} \left(x^{(i)}\right)^2}, \quad
\|x\|_\infty = \underset{1 \le i \le n}{\max} |x^{(i)}| . 
\] 
Given a function $ f $ we denote its  domain by  $\mbox{dom} f = \{x \in \mathbb{R}^n \, | \,f(x) < \infty\}. $ Further, we recall the definition of the convex conjugate of the function $f: $ \[f^\star(x^\star) = \underset{x \in \mathbb{R}^n}{\sup}\left\langle x,x^\star \right\rangle - f(x), \] where $x^\star $ is a vector of dual variables. Finally, for the $(n-1)$-dimensional simplex we write   \[
\Delta_n = \left\{p \in \mathbb{R}^n \,\left|\, \sum\limits_{i=1}^{n} p^{(i)} = 1, \;  p^{(i)} \ge 0, \; i=1, \ldots, n \right. \right\}. 
\]
The Bregman divergence of a convex function $f$ is given by:
\[
D_f(y,x) = f(y) - f(x) - \la \nabla f(x), y-x \ra, \quad \mbox{for all} \quad x,y \in \mbox{dom} f.
\]
A function $f: \R^n \rightarrow \R$ is $L$-strongly smooth w.r.t. $\|\cdot\|_\|$ norm if {it is differentiable}, and for all $x, y \in \R^n$ we have:
\[
  f(y) \leq f(x) + \langle \nabla f(x), y-x \rangle + \frac{L}{2} \|y-x\|^2.
\]
The positive constant $L$ is called the smoothness parameter of $f$.
Obviously, for a $L$-strongly smooth function $f$ it holds:
\[
D_f(y,x) \leq  \frac{L}{2} \|y-x\|^2.
\]
 
\section{Discrete Choice Review}  %% Please avoid complex formulas in (sub)titles
In this section, we review discrete choice behavior given by additive random utility models (ARUM). We argue why those models are a natural choice in order to design online optimization algortihms for the experts setting. Moreover, we present a summary of recent results concerning the algorithmic aspects of discrete choice models. In particular we focus on the connection between additive random utility models and online optimization. In online optimization data becomes available in sequential order rather than in batches. Thus, at each iteration  a new data stream arrives and an update of the decision is made. However, contrary to  multiarmed bandit models, an agent is able to observe the full vector of payoffs or losses.
\subsection{Additive Random Utility Models}\label{ss:ARUM}
 The additive decomposition of utility is motivated by psychological experiments accomplished in the 1920's \cite{thurstone}. A formal description of this framework has been first introduce in economic context \cite{gev}, where rational decision-makers choose from a finite set of mutually exclusive alternatives $\{1, \ldots, n\}$. Each alternative $i = 1,\ldots, n $ provides the utility \[u^{(i)} + \epsilon^{(i)}, \] where $ u^{(i)} \in \mathbb{R} $ is the deterministic utility part of the $i$-th alternative and $\epsilon^{(i)} $ is its stochastic error.  For the sake of clarity we use   vector notation for the deterministic utilities and the random utilities, respectively:
\[
u = \left(u^{(1)}, \ldots, u^{(n)}\right)^T, \quad \epsilon = \left(\epsilon^{(1)}, \ldots, \epsilon^{(n)}\right)^T. 
\] 
The so called surplus function of additive random utility models (ARUM) takes a central role in our work. It is the expected maximum overall utility:
\begin{equation}\label{consumer surplus general}
E(u) = \mathbb{E}_\epsilon  \left(\max_{1 \leq i \leq n} u^{(i)} + \epsilon^{(i)}\right).   
\end{equation}

The following assumption concerning random errors is standard, see e.g. \cite{palma}. 
\begin{ass}[]\label{ass:joint density}
	The random vector $\epsilon $ follows a joint distribution with zero mean that is absolutely continuous with respect to the Lebesgue measure and fully supported on $\mathbb{R}^n. $
\end{ass} 
Under Assumption \ref{ass:joint density}, the surplus function is convex and differentiable \cite{palma}. The well-known Williams-Daly-Zachary theorem  states that the gradient of E corresponds to the vector of choice probabilities \cite{gev} which can be stated in terms of partial derivatives of $E$:  \begin{equation}\label{eq:dalytheorem}
\frac{\partial E(u)}{\partial u^{(i)}} = \mathbb{P} \left( u^{(i)} +  \epsilon^{(i)} = \underset{1 \leq i \leq n}{\max} u^{(i)} + \epsilon^{(i)} \right), \quad i = 1,\ldots, n, 
\end{equation}
We denote this probability by $\mathbb{P}^{(i)}$. 
This formula holds due to Assumption \ref{ass:joint density} as  ties in Equation \eqref{consumer surplus general} occur with probability zero. 
Let us now focus on generalized extreme value models (GEV) introduced by \cite{mcfadden1978modelling}.  GEV comprise a broad class of models such as the popular multinomial logit model.  The vector $\e$ of random errors defines a generalized extreme value model (GEV) if it follows the joint distribution given by the probability density function
\[
f_\epsilon\left(y^{(1)}, \ldots, y^{(n)}\right) = \frac{\partial^n \exp\left(-G\left(e^{-y^{(1)}},\ldots, e^{-y^{(n)}}\right)\right)}{\partial y^{(1)} \cdots \partial y^{(n)}},
\]
where the generating function $G:\R^n_+ \rightarrow \R_+$ has to satisfy the following properties:
\begin{itemize}
	\item[(G1)] $G$ is homogeneous of degree $\nicefrac{1}{\mu} > 0$.
	\item[(G2)] $G\left(x^{(1)}, \ldots, x^{(i)}, \ldots, x^{(n)}\right) \rightarrow \infty$ as $x^{(i)} \rightarrow \infty$, $i=1, \ldots,n$.
	\item[(G3)] %any subset of indices $\left\{ i_1, \ldots, i_k\right\} \subset \{1,\ldots,n\}$ 
	For the partial derivatives of $G$ w.r.t. $k$ distinct variables it holds: 
	\[
	\frac{\partial^k G\left(x^{(1)},\ldots, x^{(n)}\right)}{\partial x^{\left(i_1\right)} \cdots \partial x^{\left(i_k\right)}} \geq 0 \mbox{ if } k \mbox{ is odd}, \quad 
	\frac{\partial^k G\left(x^{(1)},\ldots, x^{(n)}\right)}{\partial x^{\left(i_1\right)} \cdots \partial x^{\left(i_k\right)}} \leq 0 \mbox{ if } k \mbox{ is even}.
	\]
\end{itemize}
It is well known from \cite{mcfadden1978modelling} that the surplus function for GEV is
\begin{equation}\label{eq:gev_surplus}
E(u) = \mu \ln G\left(e^{u}\right),
\end{equation}
where we neglect an additive constant. The choice probability of the $i$-th alternative is given by the $i$-th partial derivative of the surplus function $E$:

\begin{equation}\label{eq:gev_choiceprob}
\P^{(i)} =\frac{\partial E(u)}{\partial u^{(i)}} = \mu \frac{\partial G\left(e^{u}\right)}{\partial x^{(i)}}\cdot \frac{e^{u^{(i)}}}{G\left(e^{u}\right)},
\end{equation}
An important family of GEV are the generalized nested logit (GNL) models  introduced by \cite{wen:2001}.  They are defined by the generating function

\begin{equation}\label{eq:gnl_generatingfct}
G(x)= \sum_{\ell \in L} \left( \sum_{i =1}^{n} \left(\sigma_{i\ell}\cdot x^{(i)}\right)^{\nicefrac{1}{\mu_\ell}} \right)^{\nicefrac{\mu_\ell}{\mu}}.
\end{equation}
Here, $L$ is a generic set of nests. The parameters $\sigma_{i\ell} \geq 0$ denote the shares of the $i$-th alternative with which it is attached to the $\ell$-th nest. For any fixed $i \in \{1, \ldots,n\}$ they sum up to one:
\[
\sum_{\ell \in L} \sigma_{i\ell}  = 1.
\]
$\sigma_{i\ell}=0$ means that the $\ell$-th nest does not contain the $i$-th alternative. Hence, the set of alternatives within the $\ell$-th nest is
\[
N_\ell = \left\{i \,|\, \sigma_{i\ell} >0\right\}.
\]
The nest parameters $\mu_\ell > 0$ describe the variance of the random errors while choosing alternatives within the $\ell$-th nest. Analogously, $\mu >0$ describes the variance of the random errors while choosing among the nests. For the function $G$ to fulfill (G1)-(G3) we require:
\[
\mu_\ell \leq \mu \quad \mbox{for all } \ell \in L. 
\]

We illustrate the concept of the generating function based on the multinomial logit model (MNL). Recall that in the (MNL)  the random errors in \eqref{consumer surplus general} are assumed to be IID Gumbel-distributed. 
\begin{exm}[Multinomial logit] 
\label{ex:mnl}
The generating function  
\[
G(x)= \sum_{i =1}^{n} \left(x^{(i)}\right)^{\nicefrac{1}{\mu}}
\]
leads to the multinomial logit, since The corresponding surplus function becomes
\[
E(u) = \mu \ln \sum_{i =1}^{n} e^{\nicefrac{u^{(i)}}{\mu}},
\]
and the choice probabilities are
\begin{equation}\label{eq}
\P \left( u^{(i)} + \epsilon^{(i)} = \max_{1 \leq j \leq n} u^{(j)} + \epsilon^{(j)}\right)= \frac{e^{\nicefrac{u^{(i)}}{\mu}}}{\displaystyle
	\sum_{i=1}^{n}e^{\nicefrac{u^{(i)}}{\mu}}}, \quad i=1, \ldots, n.
\end{equation}
\end{exm}
The MNL model is very popular. However, it is not able to capture non-independent substitution patterns due to the Independence of Irrelevant Alternatives Axiom (IAA). This might be a drawback in designing an online optimization algorithm for several scenarios. Another well-known instance of the GNL family which violates the IAA and is thus be able to deal with dependent alternatives  is the nested logit model.
\begin{exm}[Nested logit]
	\label{ex:nl}
	Let in GNL for every alternative $i \in \{1,\ldots,n\}$ there be a unique nest $\ell_i \in L$ with $\sigma_{i \ell_i}=1$, and $\mu=1$. Then, 
		the nests $N_\ell=\left\{i \,|\, \ell_i=\ell\right\}$ are mutually exclusive, and 
		the generating function
		\[
		G(x)=\sum_{\ell \in L} \left( \sum_{i \in N_\ell} x^{(i)\nicefrac{1}{\mu_\ell}} \right)^{\mu_\ell}
		\]
		leads to the nested logit (NL). The corresponding surplus function is
		\[
		E(u) = \mu \ln \sum_{\ell \in L} \left( \sum_{i \in N_\ell} e^{\nicefrac{u^{(i)}}{\mu_\ell}} \right)^{\mu_\ell},
		\]
		and the choice probabilities for $i \in N_\ell$, $\ell \in L$ are
		\[
		\P^{(i)} = 
		\frac{e^{\mu_\ell \ln \sum_{i \in N_\ell} e^{\nicefrac{u^{(i)}}{\mu_\ell}}}}{\displaystyle
			\sum_{\ell \in L} e^{\mu_\ell \ln \sum_{i \in N_\ell} e^{\nicefrac{u^{(i)}}{\mu_\ell}}}}\cdot   \frac{e^{\nicefrac{u^{(i)}}{\mu_\ell}}}{\displaystyle
			\sum_{i\in N_\ell} e^{\nicefrac{u^{(i)}}{\mu_\ell}}}.
		\]
	\end{exm}
\subsection{Algorithmic Aspects of ARUM and Online Optimization}
Recently, discrete choice models have been connected to  convex optimization \cite{prox}. More precisely, the authors incorporate prox-functions derived from the convex conjugates of discrete choice surplus functions  into dual averaging schemes. These results have also been applied by \cite{melo2021learning} in order to develop online optimization algorithms based on discrete choice surplus functions.  \\ Let us formally review the framework of online optimization in the $n$-experts setting. 
At each iteration $t$ an agent or learner observes a vector of  rewards $u_t \in \mathcal{Y}$, which is revealed after the agent made a decision $x_t \in \mathcal{X}$ for the $t$-th iteration.
In the $n$-experts setting where the decision space $\mathcal{X} \subset \mathbb{R}^n$ coincides with the simplex $\Delta_n$. Moreover we analyze the regret in the context of online linear optimization, where the rewards are linear.  Note that distributional assumption concerning the generated rewards is made. In fact, the rewards could be chosen adversarily.  Thus, the the regret analysis leads to robust worst-case bounds. Online optimization could also be interpreted as a repeated game between the agent and an adversarily environment. We write for the vector of cumulative  rewards $U_t = \sum_{h=1}^{t}u_h$. Then, online linear optimization can be formulated as follows: \\
For $t=1, \ldots, T$:
\begin{itemize}
 \item Agent chooses $x_t \in \Delta_n$;
 \item Adversary reveals $u_t \in \mathcal{Y}$;
 \item Agent gains $\la x_t, u_t \ra$.
\end{itemize}
In this paper we focus on scenarios where the vectors of rewards are bounded, i.\,.e we set $\mathcal{Y} = \left\{u : \|u\|_\infty \leq K\right\}$.
 In order to measure the quality of the decisions in online optimization, the notion of regret has been introduced:
 \begin{equation}\label{eq:regret}
 R_T = \max_{x \in \Delta_n} \; \la x, U_T \ra - \sum_{t=1}^{T}\la x_t, u_t \ra.
 \end{equation}
The algorithms of online optimization can be separated into two classes, Follow the Regularized Leader (FTRL) and   Follow the Perturbed Leader (FTPL), see e.\,g. \cite{abernethy2016perturbation}. Follow the Regularized Leader algorithms (FTRL)  are based on regularization techniques, well known from optimization. Hence, the regret analysis heavily relies on convex analysis tools. On the other hand, FTPL algorithms perturb the cumulative gain vector by a random variable.  \cite{abernethy2016perturbation} show that the decision variable of all algorithms of these classes can be characterized by the gradient of a scalar-valued convex potential function.  \cite{melo2021learning} proves that the surplus function of many GEV models lead to algorithms where the  regret is growing by the order $\mathcal{O}(\sqrt{T})$. 
 Thus, the average regret tends to zero, which is known as Hannan Consistency, see e.\,g. \cite{cesa2006prediction}.  The key aspect to create GBPA  from  discrete choice  models is the convex perspective of the surplus function \eqref{consumer surplus general}:
\begin{equation}\label{eq:perspective}
\tilde{E}(U;\eta) := \eta \cdot E\left(U/\eta\right), \quad \eta > 0.
\end{equation}
Rewriting Equation \eqref{eq:perspective} yields:
\[
\eta \cdot E\left(U/\eta\right) = \eta \cdot \E\left(\max_{1 \leq i \leq n} U^{(i)}/\eta + \epsilon^{(i)}\right) = \E\left(\max_{1 \leq i \leq n} U^{(i)} + \eta \cdot \epsilon^{(i)}\right),
\]
which is due to Assumption \ref{ass:joint density} a stochastic smoothing of the  $\max$-function defined by \cite{abernethy2016perturbation}. Such a surplus functions serves as potential function.  Hence, it remains to identify random utility models such that algorithms are Hannan-consistent. Our algorithms and results are very similar to  \cite{melo2021learning}, however we provide the result and a proof for several reasons. First, we examine general discrete choice models. Additionally, since we want to discuss the computational aspects of the different algorithms, a derivation of the results  regarding the regret bounds clarify  the reading.  In order to analyze general ARUM we rely on the finite modes condition from \cite{prox}.\\
\begin{dfn}
Let $g_{k,m}$ denote the density function of differences $\epsilon^{(m)} - \epsilon^{(k)}$, $k \not = m$  of random errors. Any point $\bar z_{k,m} \in \mathbb{R}$ which maximizes the density function $g_{k,m}$ is called a mode of the random variable $\epsilon^{(m)} - \epsilon^{(k)}$.
\end{dfn}
We restrict our analysis to ARUM satisfying this condition.
\begin{ass}\label{ass:agents_behavior}	 The differences $\epsilon^{(k)} - \epsilon^{(m)}$  of random errors have finite modes for all $k \not = m$.
\end{ass}
Let us state the blueprint for a GBPA based on random utility models satisfying Assumptions \ref{ass:joint density} and \ref{ass:agents_behavior}:
\begin{tcolorbox}
\begin{algo}[RUM-Algorithm for $n$-experts] {\ \\}\label{algo:olo}
 {\bfseries Input}: Discrete choice surplus function E and set of parameters $\Theta$, Stepsize $\eta >0$\\
 {\bfseries Initalize}: $U_0 = \mathbf{0}$, $x_0 = \frac{1}{n}\cdot e$\\
{\bfseries For $t=1, \ldots, T$ do}:
 \begin{itemize}
     \item Choose $x_{t} =\nabla \tilde{E}(U_{t-1};\eta)$
     \item Observe $u_{t} \in \mathcal{Y}$
     \item Receive reward $\la u_t, x_t  \ra$
     \item Update $U_t = U_{t-1} + u_t$.
 \end{itemize}

\end{algo}
\end{tcolorbox}
\begin{thm}\label{thm:regret_olo}
Let expectation of the maximum of random errors be bounded above, i.\,e. $\mathbb{E}(\max_{1 \leq i \leq n} \epsilon^{(i)} )\leq \alpha$. Then
Algorithm \ref{algo:olo} is Hannan-consistent, i.\,e.
\[
R(T) \leq  \eta\cdot \alpha + \frac{L\cdot K^2 \cdot T}{\eta}
\]
where $L = 2 \sum_{i=1}^{n}\sum_{j \neq i} g_{i,j}(\bar{z}_{i,j})$. Optimizing the scaling parameter $\eta$ yields
\[
R(T) \leq 2\cdot \sqrt{\alpha\cdot LT}\cdot K.
\]
\end{thm}
\begin{proof}
   Due to Assumption \ref{ass:joint density}, it holds that $x_t \in \mbox{rint} (\Delta_n)$ for all $t$. Consequently, Algorithm \ref{algo:olo} is an instance of the GPBA \cite{abernethy2016perturbation}. Furthermore, under Assumption \ref{ass:agents_behavior}  the surplus function is $L$-strongly smooth w.r.t. $\|\cdot \|_\infty$ \cite{prox}. Thus, the perspective $\tilde{E}$ is $\frac{L}{\eta}$-strongly smooth and the Bregman Divergence  between $U$ and $U+u$ is bounded above, i.\,e.
   \[
   \tilde{E}(U+u;\eta) - \tilde{E}(U;\eta) - \la \nabla \tilde{E}(U;\eta),u\ra \leq \frac{L}{2\eta}\cdot \|u\|^2_\infty.
   \] 
   Therefore it is justified to apply  Theorem $1.9$ of \cite{abernethy2016perturbation} which concludes the assertion.
\end{proof}
To the best of our knowledge, online optimization algorithms based on general discrete choice surplus functions have not been analyzed  in the literature before. \\ The regret bound derived in Theorem \ref{algo:olo} is strongly determined by the   the smoothness parameter of discrete choice models which depends   on the number of alternatives. In  \cite{prox} dimension-independent estimates of the smoothness parameter for several discrete choice models have been derived. In particualar,     GEV models whose generating function  $G$  satisfy the following inequality for all $x=\left(x^{(1)}, \ldots, x^{(n)}\right)^T\in \R^n_+$:
\begin{equation}\label{eq:sc.condition.gev}
\sum_{i=1}^{n} \frac{\partial^2 G(x)}{\partial x^{(i)2}} \cdot x^{(i)2} \leq M \cdot G(x),
\end{equation}
with some constant $M \in \R$. Then, the estimate of the smoothness parameter is \cite{prox}:
\begin{equation}\label{eq:smoothness_gev}
    \frac{1}{\mu} + 2 \left(\left(1-\frac{1}{\mu}\right) + \mu M \right).
\end{equation}
Moreover, the same authors prove that for the family of GNL models this condition is satisfied. This fact leads to the Hannan-consistency of GNL based online optimization algorithms as proved in \cite{melo2021learning}. For the remaining part of this section we want to focus on the computational aspects of GNL based algorithms. Clearly, the updates of Algorithm \ref{algo:olo} vary with the choice of the GNL model. The well known exponentially weighted algorithms is based on the multinomial logit model and therefore inherits IAA property which might be not desirable in situations where some of the actions are dependent. As mentioned in Section \ref{ss:ARUM} other GNL models like  the nested logit have the possibility to incorporate such dependence structure in the updates. On the same time the computational efficiency of the  updates is remained due to  the closed form given in \eqref{eq:gev_choiceprob}. The estimate of the smoothness parameter is given by \cite{prox}:  
\begin{equation}\label{eq:smoothness_gnl}
    \frac{2}{\min_{\ell \in L} \mu_\ell} -\frac{1}{\mu}.
\end{equation}
Let us further compare the nested logit  to the multinomial based algorithm.
It follows from \eqref{eq:smoothness_gnl} for the MNL algorithm $$L_{\mbox{MNL}} = \frac{1}{\mu \cdot \eta}$$ and for the nested logit\footnote{The smoothness parameter of the nested logit surplus function can be improved by the factor $\frac{1}{2}$. This is shown by the authors of \cite{dynamic} who derive the modulus of strong smoothness.} $$L_{\mbox{NL}} = \frac{2}{\min_{\ell \in L} \mu_\ell \cdot \eta}.$$
Obviously, the smoothness parameter of the MNL surplus function is better than the smoothness parameter of the NL surplus function. \\ Let us focus on the parameter $\alpha$. For that we can rely on the analysis of the function $E$, since 
\[
E(0) = \mathbb{E}(\max_{1 \leq i \leq n} \epsilon^{(i)}).
\]
Due to Equation \eqref{eq:gev_surplus} we can rewrite this as 
\[
E(\mathbf{0}) = \mu \ln G\left(e^{\mathbf{0}}\right) = \mu \ln G\left(e\right).
\]
For the MNL generating function (see Example \ref{ex:mnl}) we have:
\[
G(e) =  \sum_{i =1}^{n} \left(1\right)^{\nicefrac{1}{\mu}} = n,
\]
from where it follows that 
\begin{equation*}
E_{\mbox{MNL}}(\mathbf{0}) = \mu \cdot \ln(n).
\end{equation*}
In the case of $\eta = 1$, we have  $\alpha = \ln(n)$ which is  remarkable better than the $2 \ln(2n)$ bound derived by the moment generating function trick in \cite{abernethy2016perturbation}. Let us examine the nested logit case:
\[
G(e)= \sum_{\ell \in L} \left( \sum_{i \in N_\ell} 1^{\nicefrac{1}{\mu_\ell}} \right)^{\mu_\ell} = \sum_{\ell \in L} \left( \left|N_\ell \right|^{\mu_\ell} \right) \overset{(\star)}{\leq} \sum_{\ell \in L}   |N_\ell| = n,
\]
where in the inequality we have used  the facts that $\mu_\ell \leq 1 $ for all $\ell \in L$  and that every alternative belongs to a unique nest.
We derive a lower bound 
\begin{align*}
G(e) &= \sum_{\ell \in L} \left( \sum_{i \in N_\ell} 1^{\nicefrac{1}{\mu_\ell}} \right)^{\mu_\ell} = \sum_{\ell \in L} \left( \left|N_\ell \right|^{\mu_\ell} \right) \\ &\geq \sum_{\ell \in L} \left( \left|N_\ell \right|^{\min_{\ell \in L}\mu_\ell} \right) \overset{(\star)}{\geq} \left(\sum_{\ell \in L}  \left|N_\ell \right| \right)^{\min_{\ell \in L}\mu_\ell} \\ &\geq n^{\min_{\ell \in L}\mu_\ell}.
\end{align*}
Again, we have used the facts that $\mu_\ell \leq 1 $ for all $\ell \in L$  and that every alternative belongs to a unique nest. For inequality $(\star)$ we applied the inequality 
\[
|\sum_{i=1}^{n} x^{(i)} |^p \leq  \sum_{i=1}^{n}  |x^{(i)}|^p, \quad p \in \left(0,1\right]
\]
Altogether, this proves the following corollary:
\begin{cor}\label{cor:constants}
    For the MNL surplus function   we have  $$\alpha = \mu \cdot \ln(n).$$
    For the NL surplus function it holds:
    \[
   {\min_{\ell \in L}\mu_\ell} \cdot \ln(n) \leq \alpha \leq \ln(n).
   \]
    \end{cor}

\section{GEV Mulitarmed Bandit Algorithms}
In this section we address the adverserial multiarmed bandit setting. The goal is to generalize the Exp3-algorithm \cite{auer}, which is mainly based on the Gumbel distribution, to further GEV resp. GNL models presented in Section \ref{ss:ARUM}. Therefore, we show that surplus functions from those models can be incorporated in the  GBPA algorithm for the multiarmed bandit setting from \cite{abernethy2016perturbation}.  
In the online learning framework   the   learner gets at the $t$- round full feedback  in terms of the vector $u_t$. This means an agent is able to observe the reward of each action independent from the selected action.   In the multiarmed bandit setting however, the learner receives limited feedback. Precisely, after choosing a probability distribution over the n arms at the $t$-th iteration, one arm $i_t$ is sampled  according to the chosen distribution and only the reward $u_t^{(i_t)}$ of this sampled option is revealed. Thus, the agent has to estimate the reward vector of round $t$. This leads to the popular exploration/exploitation trade-off. Exploration means that the agent has to pull an arm in order to get information regarding its rewards. On the other hand, the learner wants to exploit  the information received so far and frequently pull the most promising arm.  Clearly, the learner's problem becomes more evolved. In the adverserial multiarmed bandit setting no distributional assumptions regarding the rewards are made, see for example \cite{slivkins}.\\ 
Apart from the EXP3, there are many algorithms tackling the multiarmed bandit such as Thompson sampling algorithms, the UCB algortihm, see \cite{lattimore} for an overview. \\     The gradients of such a potential function have to lie in the relative interior of the simplex, i.\,e. $\nabla \tilde{\Phi} \subset \mbox{rint}(\Delta_n)$. Let a potential function $\Phi$  and a potentially adversary sequence of negative rewards $u_1, u_2, \ldots, u_T \in \left[-1,0\right]^n$ be given. Then the template of the GBPA for Multi-Armed Bandits reads as (\cite{abernethy2016perturbation}):
For $t=1, \ldots, T$:
\begin{itemize}
	\item Set $\hat{U}_0 = 0$;
	\item Learner samples $i_t$ according to discrete distribution $p(\hat{U}_{t-1}) = \nabla \tilde{\Phi}(\hat{U}_{t-1})$;
	\item Learner observes and gains $u^{(i_t)}_t \in \left[-1,0\right]$;
	\item Learner estimates $\hat{u}_t := \frac{u^{(i_t)}_t}{p(\hat{U}_{t-1})}\cdot e_{i_t}$; 
	\item Update $\hat{U}_t = \hat{U}_{t-1} + \hat{u}_t$.
	
\end{itemize}
Due to the sampling process in each round, there is randomness in the performance of any algorithm. Thus, a well performing algorithm is measured w.r.t. the expected regret:
\begin{equation}\label{eq:exp_regret}
    \mathbb{E}\left[R_T\right] = \max_{1\leq i \leq n} U^{(i)}_T - \mathbb{E}\left[\sum_{t=1}^{T} \langle \nabla \tilde{\Phi}(\hat{U}_t), u_t \rangle  \right],
\end{equation}
where the expectation is taken over the agent's actions and the randomness in the environment.
The "loss only" setting is crucial in order to prove near-optimal (expected) regret bounds, see \cite{abernethy2016perturbation}. 
In what follows, we derive a class of new multi-armed bandit algorithms from ARUM. In particular, we identify GEV models which are differential-consistent according to Definition \ref{def:diff_cons}. This has several advantages. First, we provide an easy way to implement family of bandit algorithms. By easy to implement, we mean that  the sampling probabilities are given in closed form solutions. Moreover, by specifying a suitable ARUM, the learner is able take into account possible correlations  amongst the arms.  Again, the key aspect is the surplus function of GEV models. 
\begin{tcolorbox}
\begin{algo}[GEV-Algorithms for multiarmed bandits] {\ \\}\label{algo:mab}
 {\bfseries Input}: Discrete choice surplus function E and set of parameters $\Theta$, Stepsize $\eta >0$\\
 {\bfseries Initalize}: $\hat{U}_0 = 0$\\
{\bfseries For $t=1, \ldots, T$ do}:
 \begin{itemize}
     \item Sample an arm $i_{t}$ according to the distribution $x_t=\nabla \tilde{E}(\hat{U}_{t-1};\eta)$
     \item Observe and gain reward $u^{(i_t)}_{t} \in \left[-1, 0 \, \right]^n$
     \item Estimate gain vector $\hat{u}_t = \frac{u^{(i_t)}_{t}}{x^{(i_t)}_{t}} \cdot e^{(i_t)}$
     \item Update $\hat{U}_t = \hat{U}_{t-1} + \hat{u}_t$.
 \end{itemize}

\end{algo}
\end{tcolorbox}

A few remarks in order. First, it is well known that the gradient of a GEV surplus function lies in the relative interior of the probability simplex. Hence, Algorithm \ref{algo:mab} provide concrete specifications of the GBPA for the multiarmed bandit problem by \cite{abernethy2016perturbation}.  Second, Algorithm \ref{algo:mab} summarizes a large class of algorithms in a surprisingly easy manner. In fact, selecting  specific GEV model yields  different versions of Algorithm \ref{algo:mab} with different sampling probabilities.   
These versions obviously include the popular EXP3-algorithm as special case ( Example \ref{ex:mnl}). Finally, the numerical implementation can be done efficiently. Recall that the surplus function and the choice probabilities of any GEV model are given in closed forms.  Hence, we don't have to rely on techniques like geometric resampling \cite{neu2013efficient} and  the sampling and estimation steps crucially simplify. 
\begin{lem}\label{lem:mab_regret}
 The expected regret  of Algorithm  \ref{algo:mab} can be written as 
 \[
\mathbb{E}(R_T) \leq \mathbb{E}_{i_1, \ldots, i_T} \left[\sum_{t=1}^{T} \mathbb{E}_{i_t} \left[D_{\tilde{E}}\left(\hat{U_t},\hat{U_{t-1}}\right)| \hat{U_{t-1}} \right] 
\right] +  \tilde{E}\left(0; \eta\right)
 \]
\end{lem} 
\begin{proof}
We invoke    \cite[Lemma 1.12]{abernethy2016perturbation} and use the fact that  the convex perspective of the surplus function is a potential function. 
\end{proof}

The estimation of the vector $\hat{u}_t$ involves an inverse scaling by the sampling probabilities $p(\hat{U}_{t-1})$ and hence the divergence $D_{\tilde{E}}$ between $\hat{U}_{t}$ and $\hat{U}_{t-1}$ depends on the sampling probabilities. Due to this reason, the divergence can become arbitrarily large, which could lead to exploding regret bounds. 
Therefore, \cite{abernethy2016perturbation} introduce a condition for the potential function  under which this divergence can be bounded.
\begin{dfn}[Differential Consistency]
	\label{def:diff_cons}
	A convex function $f$ is $C$-differentially-consistent if there exists a constant $C >0$ such that for all $U \in (-\infty, 0)^n$ and $i=1, \ldots, n$ it holds 
	\begin{equation}\label{eq:diff_cons}
	\nabla^2_{ii}f(U) \leq C\cdot \nabla_i f(U).
	\end{equation}
\end{dfn}
For $C$-differentially-consistent potential functions    an upper bound for the divergence part of Lemma \ref{lem:mab_regret}, can be proved \cite[Theorem 1.13]{abernethy2016perturbation} i.\,e. 
\begin{equation}\label{eq:bounded_breg}
\mathbb{E}_{i_t} \left[D_{\tilde{E}}\left(\hat{U_t},\hat{U_{t-1}}\right)| \hat{U_{t-1}} \right] \leq \frac{C\cdot n}{2}, \quad t=1, \ldots, T. 
\end{equation}
As already explained, Algorithm \ref{algo:mab} is able to capture possible dependencies of action in the sampling process. Furthermore, sampling and estimation steps can be processed in a numerically highly efficient manner. Therefore, we are interested in finding GEV models such that the surplus function is $C$-differentially consistent. 
The following Theorem characterizes GEV models satisfying this property.
\begin{thm}\label{thm:mab_gev}
	Let a generating function $G$ satisfy for all $i=1, \ldots, n$, and $x=\left(x^{(1)}, \ldots, x^{(n)}\right)^T\in \R^n_+$
	\begin{equation}\label{eq:diff_cons.condition.gev}
 \frac{\partial^2 G(x)}{\partial x^{(i)2}} \cdot x^{(i)} \leq \tilde{C} \cdot \frac{\partial G(x)}{\partial x^{(i)}},
	\end{equation}
	for some constant $\tilde{C} \in \left(-1,\infty\right)$. Then, the corresponding surplus function $E$ is $C$-differentially-consistent with $C=\tilde{C} + 1$.
\end{thm}
\begin{proof}
	We have to show that Condition \eqref{eq:diff_cons} holds true for the surplus function. Thus, we need expressions for $ \frac{\partial E(u)}{\partial u^{(i)}}$ and $\frac{\partial^2 E(u)}{\partial u^{(i)2}}$. The former has been stated in Equation \eqref{eq:gev_choiceprob}. An expression for the latter has been for example given by \cite{prox}. We present both terms:
		\[
		\begin{array}{rcl}
		\displaystyle \frac{\partial E(u)}{\partial u^{(i)}} &=& \displaystyle\P^{(i)} =  \displaystyle \mu \frac{\partial G\left(e^{u}\right)}{\partial x^{(i)}}\cdot \frac{e^{u^{(i)}}}{G\left(e^{u}\right)}, \\ \\
		\displaystyle \frac{\partial^2 E(u)}{\partial u^{(i)2}}&=& \displaystyle  \frac{1}{\mu}\frac{\partial E(u)}{\partial u^{(i)}} \left( 1 - \frac{\partial E(u)}{\partial u^{(i)}}\right) + \left(1-\frac{1}{\mu}\right)\frac{\partial E(u)}{\partial u^{(i)}}
		+ \mu
		\frac{\partial^2 G\left(e^{u}\right)}{\partial x^{(i)2}}\cdot\frac{\left(e^{u^{(i)}}\right)^2}{G\left(e^{u}\right)}.
		\end{array}
		\] 
		We compute
		\begin{align*}
			\frac{\partial^2 E(u)}{\partial u^{(i)2}} &= \displaystyle  \frac{1}{\mu}\frac{\partial E(u)}{\partial u^{(i)}} \left( 1 - \frac{\partial E(u)}{\partial u^{(i)}}\right) + \left(1-\frac{1}{\mu}\right)\frac{\partial E(u)}{\partial u^{(i)}}
			+ \mu
			\frac{\partial^2 G\left(e^{u}\right)}{\partial x^{(i)2}}\cdot\frac{\left(e^{u^{(i)}}\right)^2}{G\left(e^{u}\right)} \\&= \frac{1}{\mu}\P^{(i)}\cdot \underbrace{\left(1-\P^{(i)}\right)}_{\leq 1} +\left(1-\frac{1}{\mu}\right)\P^{(i)} +  \mu
			\frac{\partial^2 G\left(e^{u}\right)}{\partial x^{(i)2}}\cdot\frac{\left(e^{u^{(i)}}\right)^2}{G\left(e^{u}\right)} \\ &\le \left(\frac{1}{\mu}+1-\frac{1}{\mu} \right)\P^{(i)} + \mu \cdot \frac{e^{u^{(i)}}}{G\left(e^{u}\right)} \cdot\frac{\partial^2 G\left(e^{u}\right)}{\partial x^{(i)2}}\cdot e^{u^{(i)}} \\ &\overset{\eqref{eq:diff_cons.condition.gev}}{\le} \P^{(i)} + \tilde{C}\cdot \mu \cdot \frac{\partial G\left(e^{u}\right) }{\partial x^{(i)}} \cdot \frac{e^{u^{(i)}}}{G\left(e^{u}\right)} = \left(1 +\tilde{C}\right)\cdot  \frac{\partial E(u)}{\partial u^{(i)}}.
		\end{align*}
		Altogether, we hence conclude that
		\[
			\frac{\partial^2 E(u)}{\partial u^{(i)2}} \le C\cdot  \frac{\partial E(u)}{\partial u^{(i)}},
		\]
		which shows the assertion.
\end{proof}
Theorem \ref{thm:mab_gev} yields a sufficient condition for GEV models to be $C$-differential-consistent. A natural question is to ask how the strong smoothness condition \eqref{eq:sc.condition.gev} is related to Condition \eqref{eq:diff_cons.condition.gev}.
\begin{prop}\label{prop:1}
	Any generating function $G$ which satisifies Condition \eqref{eq:diff_cons.condition.gev} also satisfies Condition \eqref{eq:sc.condition.gev} with $M = \frac{\tilde{C}}{\mu}$.
\end{prop}
\begin{proof}
	Let us fix any $x=\left(x^{(1)}, \ldots, x^{(n)}\right)^T\in \R^n_+$ and multiply \eqref{eq:diff_cons.condition.gev} by $ x^{(i)} \in \R^n_+$ which yields for all $i=1, \ldots, n$
	\[
	\frac{\partial^2 G(x)}{\partial x^{(i)2}} \cdot x^{(i)2} \leq \tilde{C} \cdot \frac{\partial G(x)}{\partial x^{(i)}}\cdot \cdot x^{(i)}.
	\]
	Therefore, summing up over all $i=1, \ldots, n$ does not change the inequality, i.\,e.
	\begin{equation}\label{eq:euler}
	\sum_{i=1}^{n} \frac{\partial^2 G(x)}{\partial x^{(i)2}} \cdot x^{(i)2} \leq \tilde{C} \cdot \sum_{i=1}^{n} \frac{\partial G(x)}{\partial x^{(i)}}\cdot x^{(i)}.
	\end{equation}
Due to Property $G(1)$, any generating function is $\frac{1}{\mu}$- homogeneous. Applying  Euler's theorem on homogeneous functions to the right side of \eqref{eq:euler}, see for example in \cite{pemberton2011mathematics} provides
\[
 \tilde{C} \cdot \sum_{i=1}^{n} \frac{\partial G(x)}{\partial x^{(i)}} \cdot x^{(i)} =  \tilde{C}\cdot \frac{1}{\mu} G(x).
\]
Altogether, we conclude that 
\[
		\sum_{i=1}^{n} \frac{\partial^2 G(x)}{\partial x^{(i)2}} \cdot x^{(i)2} \leq \frac{\tilde{C}}{\mu} G(x).
\]
Note that $x \in \Rn_+$ has been chose arbitrarily.
\end{proof}
Combined with the results from \cite{melo2021learning} and \cite{abernethy2016perturbation} Proposition \ref{prop:1} states the class of GEV models that can be used for full information online optimization is at least as large as the class of GEV models suitable for bandit algorithms.\\ \vspace*{0.5 cm}

Let us turn our attention to the generalized nested logit models introduced by \cite{wen:2001} and recall the generating function presented in Equation \eqref{eq:gnl_generatingfct}:
\[
G(x)= \sum_{\ell \in L} \left( \sum_{i =1}^{n} \left(\sigma_{i\ell}\cdot x^{(i)}\right)^{\nicefrac{1}{\mu_\ell}} \right)^{\nicefrac{\mu_\ell}{\mu}}.
\]

Let us analyze the $C$-differential-consistency of GNL models.
\begin{thm}\label{cor:mab_gnl}
	For GNL the corresponding surplus function is $\frac{1}{\displaystyle\min_{\ell \in L} \mu_\ell}$- differential-consistent.
\end{thm}
\begin{proof}
	We review the following formulas, which were derived in the proof of Corollary 4 by \cite{prox}:
	 \[
	 \frac{\partial G\left(x\right)}{\partial x^{(i)}}= \frac{1}{\mu}\sum_{\ell \in L} \left( \sum_{i =1}^{n} \left(\sigma_{i\ell}\cdot x^{(i)}\right)^{\nicefrac{1}{\mu_\ell}} \right)^{\nicefrac{\mu_\ell}{\mu}-1}
	 \left(\sigma_{i\ell}\cdot x^{(i)}\right)^{\nicefrac{1}{\mu_\ell}-1} \cdot \sigma_{i\ell},
	 \]
	 and
	 \[
	 \begin{array}{rcl}
	 \displaystyle \frac{\partial^2 G(x)}{\partial x^{(i)2}} &=& \displaystyle \frac{1}{\mu}\sum_{\ell \in L} 
	 \frac{1}{\mu_\ell}\left(\frac{\mu_\ell}{\mu}-1\right)
	 \left( \sum_{i =1}^{n} \left(\sigma_{i\ell}\cdot x^{(i)}\right)^{\nicefrac{1}{\mu_\ell}} \right)^{\nicefrac{\mu_\ell}{\mu}-2}
	 \left(\left(\sigma_{i\ell}\cdot x^{(i)}\right)^{\nicefrac{1}{\mu_\ell}-1} \cdot\sigma_{i\ell}\right)^2 \\ \\
	 &+& \displaystyle \frac{1}{\mu}\sum_{\ell \in L} \left(\frac{1}{\mu_\ell}-1 \right) \left( \sum_{i =1}^{n} \left(\sigma_{i\ell}\cdot x^{(i)}\right)^{\nicefrac{1}{\mu_\ell}} \right)^{\nicefrac{\mu_\ell}{\mu}-1}
	 \left(\sigma_{i\ell}\cdot x^{(i)}\right)^{\nicefrac{1}{\mu_\ell}-2}  \cdot\sigma_{i\ell}^2.
	 \end{array}
	 \]
	 Due to $\mu_\ell \leq \mu$, $\ell \in L$, it holds:
	 \[
	 \frac{\partial^2 G(x)}{\partial x^{(i)2}} \leq \frac{1}{\mu}\sum_{\ell \in L} \left(\frac{1}{\mu_\ell}-1 \right) \left( \sum_{i =1}^{n} \left(\sigma_{i\ell}\cdot x^{(i)}\right)^{\nicefrac{1}{\mu_\ell}} \right)^{\nicefrac{\mu_\ell}{\mu}-1}
	 \left(\sigma_{i\ell}\cdot x^{(i)}\right)^{\nicefrac{1}{\mu_\ell}-2} \cdot\sigma_{i\ell}^2.
	 \]
	 We multiply by $x^{(i)}$ and get 
	 \[
	 \frac{\partial^2 G(x)}{\partial x^{(i)2}}\cdot x^{(i)} \leq \frac{1}{\mu}\sum_{\ell \in L} \left(\frac{1}{\mu_\ell}-1 \right) \left( \sum_{i =1}^{n} \left(\sigma_{i\ell}\cdot x^{(i)}\right)^{\nicefrac{1}{\mu_\ell}} \right)^{\nicefrac{\mu_\ell}{\mu}-1}
	 \left(\sigma_{i\ell}\cdot x^{(i)}\right)^{\nicefrac{1}{\mu_\ell}-1} \cdot\sigma_{i\ell}.
	 \]
	 We follow similar considerations as \cite{prox} and conclude
	  \begin{align*}
	  \frac{\partial^2 G(x)}{\partial x^{(i)2}}\cdot x^{(i)} &\leq \frac{1}{\mu}\sum_{\ell \in L} \left(\frac{1}{\mu_\ell}-1 \right) \left( \sum_{i =1}^{n} \left(\sigma_{i\ell}\cdot x^{(i)}\right)^{\nicefrac{1}{\mu_\ell}} \right)^{\nicefrac{\mu_\ell}{\mu}-1}
	  \left(\sigma_{i\ell}\cdot x^{(i)}\right)^{\nicefrac{1}{\mu_\ell}-1} \cdot\sigma_{i\ell} \\ &\leq\displaystyle\max_{\ell \in L} \left(\frac{1}{\mu_\ell}-1\right) \cdot \frac{1}{\mu}\sum_{\ell \in L}\left( \sum_{i =1}^{n} \left(\sigma_{i\ell}\cdot x^{(i)}\right)^{\nicefrac{1}{\mu_\ell}} \right)^{\nicefrac{\mu_\ell}{\mu}-1}
	  \left(\sigma_{i\ell}\cdot x^{(i)}\right)^{\nicefrac{1}{\mu_\ell}-1} \cdot\sigma_{i\ell} \\ &=  \left(\frac{1}{\displaystyle\min_{\ell \in L}\mu_\ell}-1\right) \cdot \frac{\partial G\left(x\right)}{\partial x^{(i)}}.
	  \end{align*}
	 Consequently, we can set $\tilde{C} = \left(\frac{1}{\displaystyle\min_{\ell \in L}\mu_\ell}-1\right)$. It remains to apply Theorem \ref{thm:mab_gev} which concludes the assertion by yielding $C = \frac{1}{\displaystyle\min_{\ell \in L}\mu_\ell}$. 
\end{proof}
Theorem \ref{cor:mab_gnl} enables to apply  the family of GNL models in the adverserial bandit setting.   Note that this family not only contains the multinomial logit with independent arms but also several models which are able to incorporate correlation structure such as nested logit, paired combinatorial logit \ldots . \\ In \cite{prox} the constant $M$ from Condition \eqref{eq:sc.condition.gev} is derived for GNL models, i.\,e. $M = \frac{1}{\mu}\left(\frac{1}{\displaystyle \min_{\ell \in L}\mu_\ell} -1 \right) $.  Considering Proposition \ref{prop:1} we hence see that $M=\frac{\tilde{C}}{\mu}$. Furthermore, the constant $C$ which enters the (expected) regret bound, only depends on the smallest nest parameter. Let us illustrate the constant $C$ based on the examples from Section \ref{ss:ARUM}.
\begin{rem}[Differential-Consistency of MNL and NL]
Recall the MNL model and its generating function from Example \ref{ex:mnl}.
	Note that in this example we have $\tilde{C} = \left(\frac{1}{\mu}-1\right)$ and therefore $C = \frac{1}{\mu}$. \\[0.5 em]
For the NL model from  Example \ref{ex:nl} we have
 $\tilde{C} = \left(\frac{1}{\displaystyle \min_{\ell \in L}\mu_\ell} -1 \right)$ and therefore $C = \frac{1}{\displaystyle \min_{\ell \in L}\mu_\ell}$. 
\end{rem}

We can finally state the main result of this Section.
\begin{thm}\label{thm:main}
    The Algorithm \ref{algo:mab} with a surplus function following a Generalized Nested Logit model  is at most 
    \[
\eta \cdot E(\mathbf{0}) + \frac{n \cdot T}{\displaystyle \min_{\ell \in L \cdot \eta} \mu_\ell}.
 \]
\end{thm}
\begin{proof}
    We apply Lemma \ref{lem:mab_regret} and conclude that $\tilde{E}(\mathbf{0};\eta) = \eta \cdot E(\mathbf{0})$. Furthermore, due to Theorem \ref{cor:mab_gnl} the surplus function is $\frac{1} \mu$- differentiable consistent and thus, its convex perspective is $\frac{1}{\eta\cdot \displaystyle \min_{\ell \in L}\mu_\ell}$- differentiable consistent. Together with Inequality \eqref{eq:bounded_breg} this provides an upper bound of  $\frac{n \cdot T}{\displaystyle \min_{\ell \in L \cdot \eta} \mu_\ell}$  for the divergence part, which concludes the assertion.
\end{proof}
 GNL models can not only  be used to design algorithms for online linear optimization algorithms but also for adversarial multiarmed bandit problems. This result enable the learner to design a large amount of computationally efficient algorithms with vanishing average regret and with sampling probabilities adjusted to the dependence structure of the arms.  
 \section{ Numerical Experiments for the Stochastic Multiarmed Bandits}
 In this section we want to compare the numerical performance of Algorithm \ref{algo:mab} in the stochastic bandit environment with nonnegative rewards.  In particular, we compare Algorithm \ref{algo:mab} with a NL surplus function to the MNL suprlus function which basically coincides with the EXP3-Algorithm. 
 Recall that the proof of Theorem \ref{thm:mab_gev} heavily rely on the concept of differential consistency introduced by \cite{abernethy2016perturbation}. Hence, the vanishing average regret is guaranteed in the loss only setting.  Hence, the question arises if instances of Algorithm \ref{algo:mab} apart from the EXP3-Algorithm can be applied to scenarios where the rewards are nonnegative. For that we provide numerical results for  a stochastic bandit setting, which is a collection of distributions, see for example \cite{lattimore}. The difference to the adversarial bandit setting lies in the assumption concerning how the sequence of rewards is generated. In the stochastic setting, after the agent has selected a distribution over the arms and sampled one arm according to this distribution, the environment samples a reward from the respective reward-generating distribution. In particular, we focus on the Bernoulli-Bandit setting meaning there exists a vector $\pi$ with elements $\pi^{(i)} \in [0,1]$, $i=1, \ldots, n$. The entries of this parameter vector determine the probability with which the agent could expect to receive a reward if playing the corresponding arm.   Consequently, the reward of the $i$-th arm  $u^{(i)}_t$  in each round is either  $1$ with probability  $\theta^{(i)}$ or $0$ with probability  $1-\theta^{(i)}$. Note that in the stochastic setting, the  performance of the best arm is random.  Clearly, a good strategy for the learner is therefore to follow the arm with the highest reward generating probability resp. mean. Thus, the learner has to balance between exploiting the arm with the highest mean so far and exploiting other arms to possibly find higher means.   For our numerical simulations the  average (expected) regret is therefore computed by 
 \begin{equation}
    \mathbb{E}\left[\mbox{Regret}_T\right] = \max_{1\leq i \leq n} \pi^{(i)} - \frac{1}{T}\mathbb{E}\left[\sum_{t=1}^{T} u^{(i_t)}_{t} \right],
 \end{equation}
 where the expectation is taken over  the randomness in the reward generating and the sampling distribution.
 We compare results of the well known Exp3-Algorithm to the algorithm based on Nested Logit probabilities. For the comparison an environment of $K$-Bernoulli-arms is initialized with corresponding parameters $\pi_k$, $k=1, \ldots, K$. Then, each algorithm runs for $T=10000$ iterations, which is repeated $B=100$ times. The parameter $\eta$ is set to $1$. We equivalently control the exploration-exploitation  of the MNL-Algorithm by the parameter $mu$.   \\
 The first environment is summarized in Table \ref{tab:env_1} 

 \begin{table}
     \centering
     \begin{tabular}{|c|c|}
     \hline
           Arm $k$ & $\pi_k$  \\
          $1$& $0.2$ \\
           $2$&  $0.8$ \\
          $3$&  $0.87$ \\
           $4$& $0.15$\\
           \hline
     \end{tabular}
     \caption{Environment 1}
     \label{tab:env_1}
 \end{table}
  We run the MNL-Algorithm with  $\mu =0.25$.For the NL-Algorithm we put alternatives $1$ and $3$ in one nest with $\mu_1=0.05$ and the two other arms in the second nest with $\mu_2=0.1$. The average (expected) regret is quickly vanishing for both algorithms, see Figure \ref{fig:avg_regret_sim1}.
\begin{figure}
     \centering
     \includegraphics[width=0.8\linewidth]{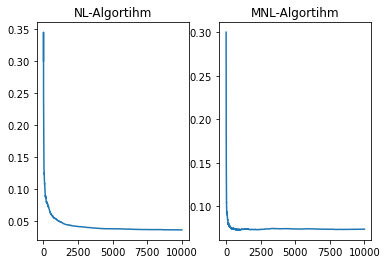}
     \caption{Average Expected Regret on Environment 1}
     \label{fig:avg_regret_sim1}
 \end{figure}
 In fact, the NL-Algortihms outperforms the MNL-Algorithm on this environment w.r.t. the average regret. This can also be seen by comparing the total average reward of the algorithms on the $100$ runs. The NL-Algorithm gains in average a reward $8337.86$ while the MNL-Algorithm receives $7962.24$ on average. In terms of exploiting the former thus performs very well. Let us inspect the exploration structure in Figures \ref{fig:lp_sim1} and \ref{fig:pa_sim1}. 
\begin{figure}
    \centering
    \includegraphics[width=0.8\linewidth]{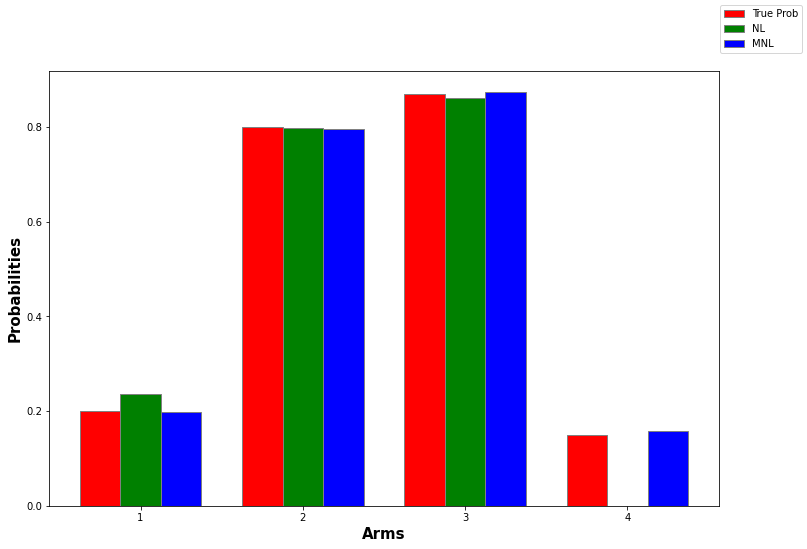}
    \caption{Learnt Reward Probabilities of Environment $1$ with $\mu = 0.25$,$\mu_1 = 0.05$, $\mu_2 = 0.1$}
    \label{fig:lp_sim1}
\end{figure}
\begin{figure}
    \centering
    \includegraphics[width=0.9\linewidth]{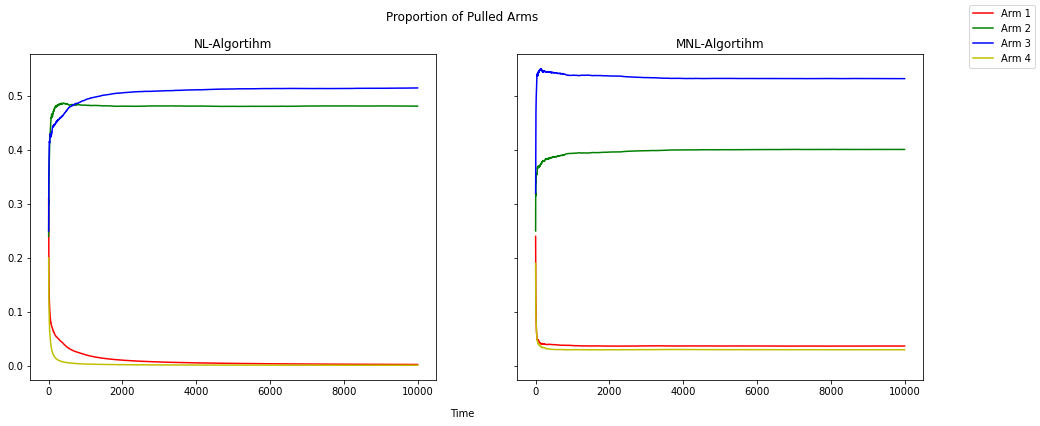}
    \caption{Played Arms of Environment $1$}
    \label{fig:pa_sim1}
\end{figure}
We see that the MNL-Algorithm learns all reward probabilities whereas the NL-Algorithm neglects Arm $4$, due to the fact that this arm is hardly played. This indicates a powerful feature of the NL-Algorithm compared to the MNL-Algortithm. The latter is able to exploit more by adjusting the smoothness parameter $\mu$, however, the exploration will suffer from this. We run a MNL-Algorithm with $\mu = 0.05$ on the same environment which gains an average reward of $8413.37$, but only explored the reward probability of Arm $3$. This means high exploitation but almost no exploration. On the contrary, readjusting the parameters of the NL-Algorithm to $\mu_1 = 0.15$ and $\mu_2 = 0.2$ yields an average reward of $8202.71$ while learning the reward probabilities of all arms, see Figure \ref{fig:sim4}.
\begin{figure}
    \centering
    \includegraphics[width=0.8\linewidth]{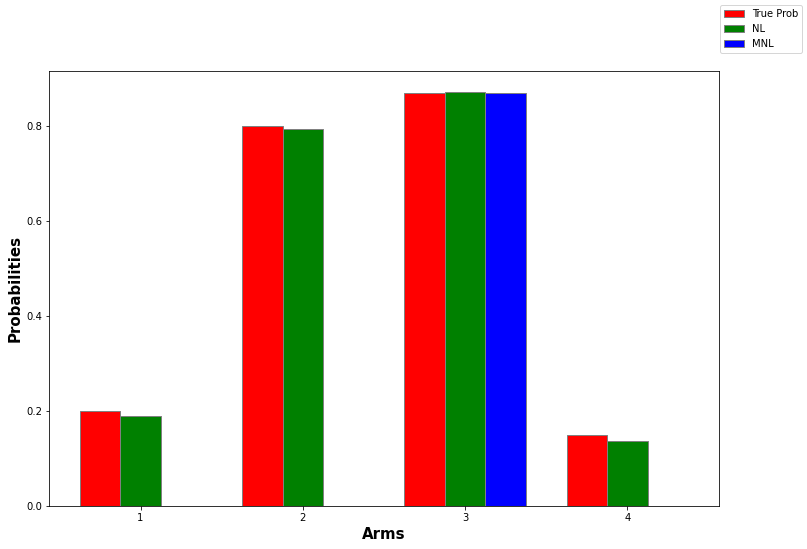}
    \caption{Learnt Reward Probabilities of Environment $1$ with $\mu=0.05, \mu_1 = 0.15, \mu_2 = 0.2$.}
    \label{fig:sim4}
\end{figure}
Hence, the NL-Algorithm provides the opportunity to attain a better Exploration/Exploitation Trade-Off by fine-tuning the nest parameters. \\ Since the MNL model can be viewed as a special case of the NL model, we could tune the NL-Algorithm in order to reflect the results of the MNL algorithm by setting the nest parameters very close to $1$. This is displayed in Figures \ref{fig:sim3_pl} and \ref{fig:sim3_pa}. Non surprisingly the average rewards of the $100$ simulations is almost the same with $6112.89$ for NL and $6108.29$ for the MNL-Algorithm. \\[0.5 em]

\begin{figure}
    \centering
    \includegraphics[width=0.8\linewidth]{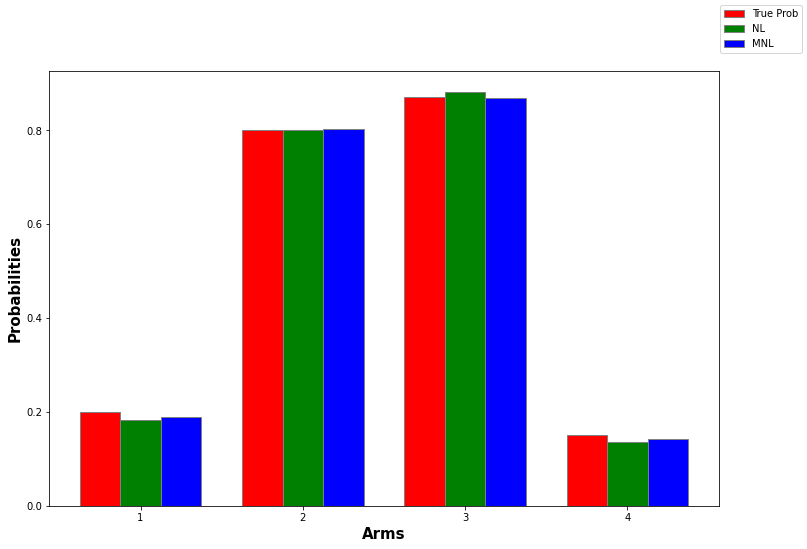}
    \caption{Learnt Reward Probabilities of Environment $1$, $\mu = 1, \mu_1 = \mu_2 = 0.998$}
    \label{fig:sim3_pl}
\end{figure}
\begin{figure}
    \centering
    \includegraphics[width=0.9\linewidth]{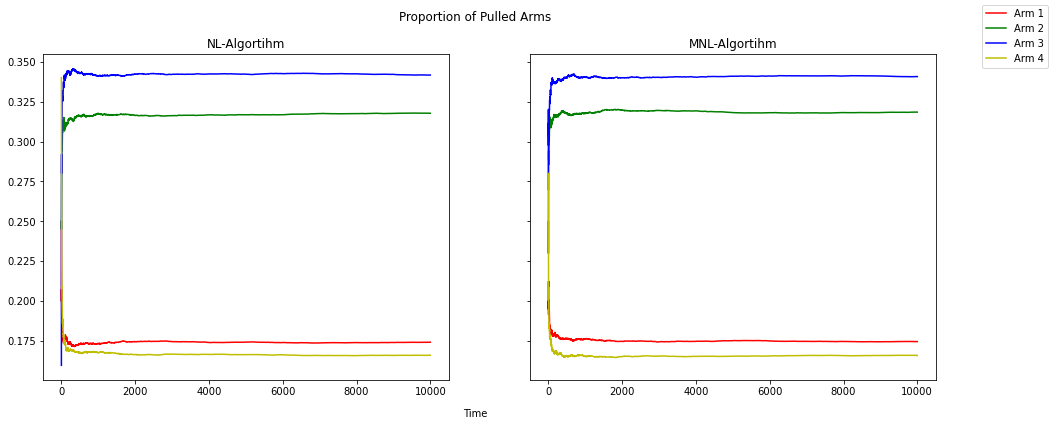}
    \caption{Played Arms of Environment $1$, $\mu = 1, \mu_1 = \mu_2 = 0.998$}
    \label{fig:sim3_pa}
\end{figure}

In order to verify the results from the simulations on Environment $1$ we compare the algorithms on a larger environment with $13$ Bernoulli arms. The corresponding mean reward parameter are displayed in Table \ref{tab:env2}.

\begin{table}
    \centering
    \begin{tabular}{|c|c|}
    \hline
         Arm $k$ & $\pi_k$  \\
        $1$ & $0.2$\\ 
        $2$& $0.3$\\
        $3$& $0.87$\\
        $4$  & $0.15$ \\
        $5$ & $0.79$ \\
        $6$ & $0.12$\\
        $7$& $0.85$\\
        $8$ & $0.1$\\
        $9$ & $0.83$\\
        $10$ & $0.75$ \\
        $11$ & $0.14$\\
        $12$ & $0.9$\\
        $13$ & $0.2$\\ \hline
    \end{tabular}
    \caption{Environment $2$}
    \label{tab:env2}
\end{table}

Alternatives $1$ to $6$ are in the first nest with corresponding parameter $\mu_1 = 0.16$. The second nest consists of arms $7$ and $8$ with $\mu_2 = 0.09$. Nest § with parameter $\mu_3 = 0.21$ includes arms $9$ to $11$. The last two arms are in the $4$-th nest with $\mu_4 = 0.12$. Again we compare to a MNL-Algorithm with parameter $\mu = 0.25$. The average (expected) regret as well as the learnt reward probabilities  can be seen in  Figures \ref{fig:regret_env2} and \ref{fig:lp_env2}.
\begin{figure}
    \centering
    \includegraphics[width=0.8\linewidth]{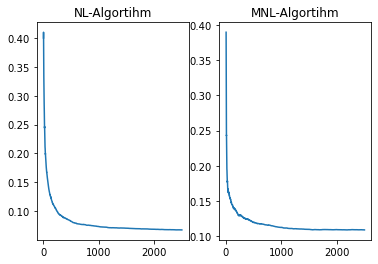}
    \caption{Average Expected Regret on Environment $2$}
    \label{fig:regret_env2}
\end{figure}
\begin{figure}
    \centering
    \includegraphics[width=0.8\linewidth]{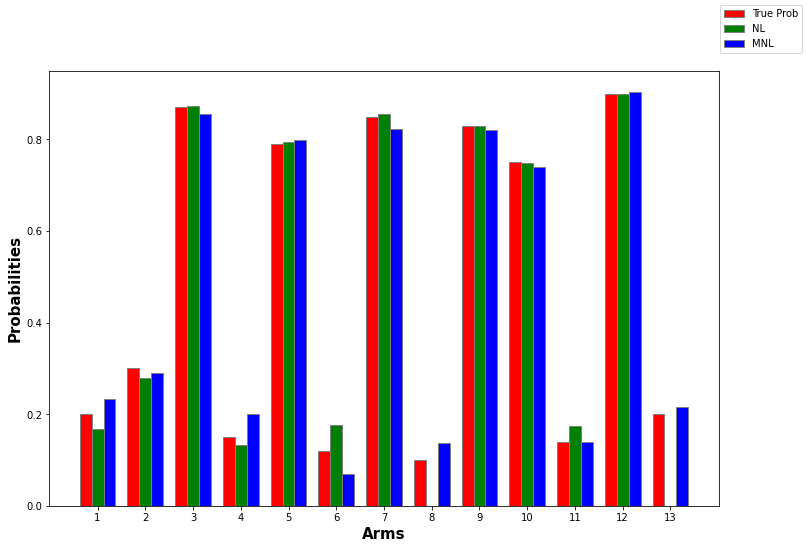}
    \caption{Learnt Reward Probabilities  of Environment $2$}
    \label{fig:lp_env2}
\end{figure}

Again we see that the NL-Algorithm is exploring most of the arms and at the same time an average reward of $8360.73$ has been gained compared to an average reward of $7911.88$. \\[0.5 em]
Our simulations show that the extension of multiarmed bandit algorithms from a multinomial based algorithm to generalized nested logit algorithm can be of high practical benefit. Indeed, the higher flexibility in tuning the parameters of the algorithm provide more possibilities to outbalance the exploration/exploitation trade-off. In particular, the algorithm could be used in scenarios, where the learner wants to explore different product categories and at the same time, wants to avoid exploring weaker products of each category.

\bibliographystyle{plainnat}
\bibliography{lit}
\end{document}